\documentclass[conference]{IEEEtran}

\usepackage{amsmath,amssymb,amsthm}
\usepackage{booktabs}
\usepackage{graphicx}
\usepackage{placeins}
\usepackage{cite}
\usepackage{url}
\usepackage{tikz}
\usetikzlibrary{positioning,arrows.meta,decorations.pathmorphing}

\newtheorem{theorem}{Theorem}

\newtheorem{proposition}[theorem]{Proposition}

\newtheorem{example}[theorem]{Example}

\newcommand{\R}{\mathbb{R}}
\newcommand{\E}{\mathbb{E}}
\newcommand{\Prob}{\mathbb{P}}

\newcommand{\FIM}{\mathcal{I}}
\newcommand{\Gram}{\mathbf{G}}
\DeclareMathOperator{\rank}{rank}

\DeclareMathOperator{\supp}{supp}

\DeclareMathOperator{\diag}{diag}

\title{Fisher Simplicity in Kolmogorov-Arnold Networks and Multilayer Perceptrons}

\author{
\IEEEauthorblockN{Ami Tavory}
\IEEEauthorblockA{Meta Platforms\\
Tel-Aviv, Israel\\
\texttt{atavory@meta.com}}
\and
\IEEEauthorblockN{Meir Feder}
\IEEEauthorblockA{School of Electrical and Computer Engineering\\
Tel-Aviv University\\
Tel-Aviv, Israel\\
\texttt{meir@tau.ac.il}}
}

\begin{document}
\maketitle

  \begin{abstract}
  Kolmogorov-Arnold Networks (KANs) are motivated in part by interpretability:
  their learned edge functions can be inspected, pruned, and reduced to symbolic
  structure.  In a fixed-basis KAN, this makes a small or zero basis coefficient
  look like a certificate of simplicity, much as a dead rectified linear unit
  (ReLU) marks unused computation in a multilayer perceptron (MLP).  Fisher
  nullity gives a precise statistical notion: a parameter direction is
  Fisher-simple exactly when perturbing it is invisible under the task
  distribution.  We study when these architectural and Fisher notions agree.

  A dead ReLU unit closes a path and makes its associated score directions
  vanish.  A zero KAN coefficient, like a zero MLP weight, need not close a path
  and therefore need not be Fisher-null.  In the single-layer Gaussian case,
  the coefficient Fisher matrix is a basis Gram matrix under the input
  distribution and is independent of the fitted coefficients.  In a multilayer
  KAN, Fisher simplicity is graph-path based:
  the data must reach a basis atom and its perturbation must propagate through
  the downstream network.  We encode these two conditions in an effective edge
  measure and, under local dictionary independence and effective-measure
  nondegeneracy, show that zero effective exposure exactly identifies
  Fisher-null directions within an edge.  Controlled synthetic diagnostics
  illustrate that zero coefficients can preserve rank while effective path
  disconnections remove the predicted directions.  These diagnostics concern
  local Fisher geometry, not end-to-end pruning quality; coefficient magnitude
  alone is not a Fisher-based pruning criterion.
  \end{abstract}

\section{Introduction}
\label{sec:intro}

Neural networks come with architecture-specific notions of simplicity.  In a
rectified linear unit (ReLU) multilayer perceptron (MLP), a hidden unit is called
dead when its activation vanishes on the task
distribution~\cite{hu2016trimming,lu2020dying}.  Kolmogorov-Arnold Networks
(KANs) place learned univariate functions on edges.  Their interpretability
motivation comes from making those functions inspectable, sparsifying and
pruning the graph, and extracting symbolic structure from the learned
model~\cite{liu2024kan,liu2025kan2}.  In a fixed-basis KAN, the apparent
coefficient-level analogue of a dead activation is a basis component whose
coefficient is small or zero.  These signals need not describe the same
statistical event: a dead activation closes a path, whereas a small or zero
parameter value, in either architecture, need not do so.

\bigskip

The Fisher information matrix (FIM) supplies a precise notion of simplicity.
A parameter direction is Fisher-simple precisely when perturbing it is
statistically invisible on the task distribution, or equivalently Fisher-null.
We ask when the usual
architectural signals above coincide with this exact loss of a local
statistical direction.

\bigskip

For a dead ReLU unit, the two notions agree.  The closed activation region
zeros both the unit and its local derivative on the task support, forcing all
associated score directions to vanish.

\smallskip

For fixed-basis KANs, the notions separate.  We study edge functions that are
linear combinations of B-splines on a fixed grid and, optionally, one sigmoid
linear unit (SiLU) atom~\cite{liu2024kan,deBoor1978}; any fixed spline scale is
absorbed into the coefficients.  Let $X\sim P_X$ be the task input,
$\phi(X)$ the fixed-basis feature vector, and $\FIM_\phi$ its coefficient
Fisher matrix.  With fixed Gaussian observation variance~$\sigma^2$, a
one-layer KAN has
\[
\FIM_\phi=\sigma^{-2}\E[\phi(X)\phi(X)^\top].
\]
The coefficient values are absent.  A zero coefficient therefore does not make
its direction Fisher-null.  In a multilayer KAN, Fisher simplicity is instead
graph-path based: a basis atom is Fisher-visible only when the task distribution
reaches its support and its perturbation propagates through the downstream
network.  Section~\ref{sec:fixed-fisher-geom} encodes these two requirements in
an edge-specific effective distribution.

\medskip

Figure~\ref{fig:fisher-null-events} summarizes how these two common simplicity
signals map to the same tangent-space criterion.

\newcommand{\FisherNullMLPPanel}[3]{%
  \begin{scope}[shift={(#1,#2)},scale=#3,every node/.append style={transform shape}]
    \node[paneltitle] at (0,4.15) {A. MLP dead unit};

    \node[graynode] (mlpIone) at (-1.10,0) {};
    \node[graynode] (mlpItwo) at ( 1.10,0) {};

    \node[graynode] (mlpAone)  at (-1.10,1.10) {};
    \node[deadnode] (mlpAdead) at ( 0.00,1.10) {};
    \node[graynode] (mlpAthree) at (1.10,1.10) {};

    \node[graynode] (mlpBone)   at (-0.95,2.35) {};
    \node[graynode] (mlpBtwo)   at ( 0.00,2.35) {};
    \node[graynode] (mlpBthree) at ( 0.95,2.35) {};

    \node[graynode] (mlpO) at (0,3.40) {};

    \draw[graywire] (mlpIone) -- (mlpAone);
    \draw[graywire] (mlpIone) -- (mlpAdead);
    \draw[graywire] (mlpIone) -- (mlpAthree);
    \draw[graywire] (mlpItwo) -- (mlpAone);
    \draw[graywire] (mlpItwo) -- (mlpAdead);
    \draw[graywire] (mlpItwo) -- (mlpAthree);

    \draw[graywire] (mlpAone) -- (mlpBone);
    \draw[graywire] (mlpAone) -- (mlpBtwo);
    \draw[graywire] (mlpAone) -- (mlpBthree);
    \draw[graywire] (mlpAdead) -- (mlpBone);
    \draw[graywire] (mlpAdead) -- (mlpBtwo);
    \draw[graywire] (mlpAdead) -- (mlpBthree);
    \draw[graywire] (mlpAthree) -- (mlpBone);
    \draw[graywire] (mlpAthree) -- (mlpBtwo);
    \draw[graywire] (mlpAthree) -- (mlpBthree);

    \draw[graywire] (mlpBone) -- (mlpO);
    \draw[graywire] (mlpBtwo) -- (mlpO);
    \draw[graywire] (mlpBthree) -- (mlpO);

    \draw[black,line width=.7pt] (-.09,1.01) -- (.09,1.19);
    \draw[black,line width=.7pt] (-.09,1.19) -- (.09,1.01);

    \node[panelnote] at (0,-1.06) {dead hidden unit\\score columns zero};
  \end{scope}%
}

\newcommand{\FisherNullKANPanel}[8]{%
  \begin{scope}[shift={(#1,#2)},scale=#3,every node/.append style={transform shape}]
    \node[paneltitle] at (0,4.15) {#6};

    \node[graynode] (#4Ione) at (-1.05,0) {};
    \node[graynode] (#4Itwo) at ( 1.05,0) {};

    \node[graynode] (#4Hone)   at (-1.50,2.05) {};
    \node[graynode] (#4Htwo)   at (-0.50,2.05) {};
    \node[graynode] (#4Hthree) at ( 0.50,2.05) {};
    \node[graynode] (#4Hfour)  at ( 1.50,2.05) {};

    \node[graynode] (#4O) at (0,3.35) {};

    \node[graybox]  (#4Bone)   at (-1.75,1.02) {};
    \node[#5]       (#4Bhot)   at (-1.25,1.02) {};
    \node[grayboxc] (#4Bthree) at (-0.75,1.02) {};
    \node[grayboxb] (#4Bfour)  at (-0.25,1.02) {};
    \node[graybox]  (#4Bfive)  at ( 0.25,1.02) {};
    \node[grayboxb] (#4Bsix)   at ( 0.75,1.02) {};
    \node[grayboxc] (#4Bseven) at ( 1.25,1.02) {};
    \node[graybox]  (#4Beight) at ( 1.75,1.02) {};

    \node[graybox]  (#4Tone)   at (-1.50,2.70) {};
    \node[grayboxb] (#4Ttwo)   at (-0.50,2.70) {};
    \node[grayboxc] (#4Tthree) at ( 0.50,2.70) {};
    \node[graybox]  (#4Tfour)  at ( 1.50,2.70) {};

    \draw[graywire] (#4Ione) -- (#4Bone.south);    \draw[graywire] (#4Bone.north) -- (#4Hone);
    \draw[#8]       (#4Ione) -- (#4Bhot.south);    \draw[#8]       (#4Bhot.north) -- (#4Htwo);
    \draw[graywire] (#4Ione) -- (#4Bthree.south);  \draw[graywire] (#4Bthree.north) -- (#4Hthree);
    \draw[graywire] (#4Ione) -- (#4Bfour.south);   \draw[graywire] (#4Bfour.north) -- (#4Hfour);

    \draw[graywire] (#4Itwo) -- (#4Bfive.south);   \draw[graywire] (#4Bfive.north) -- (#4Hone);
    \draw[graywire] (#4Itwo) -- (#4Bsix.south);    \draw[graywire] (#4Bsix.north) -- (#4Htwo);
    \draw[graywire] (#4Itwo) -- (#4Bseven.south);  \draw[graywire] (#4Bseven.north) -- (#4Hthree);
    \draw[graywire] (#4Itwo) -- (#4Beight.south);  \draw[graywire] (#4Beight.north) -- (#4Hfour);

    \draw[graywire] (#4Hone) -- (#4Tone.south);    \draw[graywire] (#4Tone.north) -- (#4O);
    \draw[#8]       (#4Htwo) -- (#4Ttwo.south);    \draw[#8]       (#4Ttwo.north) -- (#4O);
    \draw[graywire] (#4Hthree) -- (#4Tthree.south);\draw[graywire] (#4Tthree.north) -- (#4O);
    \draw[graywire] (#4Hfour) -- (#4Tfour.south);  \draw[graywire] (#4Tfour.north) -- (#4O);

    \node[panelnote] at (0,-1.06) {#7};
  \end{scope}%
}

\newcommand{\FisherNullKANAliveWitness}[3]{%
  \begin{scope}[shift={(#1,#2)},scale=#3,every node/.append style={transform shape}]
    \draw[black!45,line width=.45pt] (-1.12,0) -- (1.12,0);
    \draw[black,line width=.55pt] (-.38,-.07) -- (-.38,.07);
    \draw[black,line width=.55pt] ( .38,-.07) -- ( .38,.07);
    \draw[black!55,line width=.55pt,densely dashed]
      (-.38,.02) .. controls (-.22,.43) and (.22,.43) .. (.38,.02);
    \foreach \x in {-.30,-.17,-.04,.11,.26} {
      \draw[black,line width=.55pt] (\x,-.12) -- (\x,.12);
    }
    \draw[black,line width=1.15pt,line cap=round] (-.38,-.20) -- (.38,-.20);
    \node[witnesslabel] at (0,.56) {$b_i$ exposed};
    \node[witnesslabel] at (0,-.42) {$W_e b_i\neq0,\ \rho_i^{(e)}>0$};
  \end{scope}%
}

\newcommand{\FisherNullKANDeadWitness}[3]{%
  \begin{scope}[shift={(#1,#2)},scale=#3,every node/.append style={transform shape}]
    \draw[black!45,line width=.45pt] (-1.12,0) -- (1.12,0);
    \draw[black,line width=.55pt] (-.38,-.07) -- (-.38,.07);
    \draw[black,line width=.55pt] ( .38,-.07) -- ( .38,.07);
    \draw[black!55,line width=.55pt,densely dashed]
      (-.38,.02) .. controls (-.22,.43) and (.22,.43) .. (.38,.02);
    \foreach \x in {-.96,-.78,-.61,.62,.80,.98} {
      \draw[black!35,line width=.55pt] (\x,-.12) -- (\x,.12);
    }
    \draw[black!22,line width=1.15pt,line cap=round] (-.38,-.20) -- (.38,-.20);
    \draw[black,line width=.70pt,line cap=round] (-.48,.36) -- (.48,-.24);
    \draw[black,line width=.70pt,line cap=round] (-.48,-.24) -- (.48,.36);
    \node[witnesslabel] at (0,.56) {$b_i$ unexposed};
    \node[witnesslabel] at (0,-.42) {$W_e b_i=0\ {\rm a.s.},\ \rho_i^{(e)}=0$};
  \end{scope}%
}

\begin{figure}[t]
\centering
\resizebox{\columnwidth}{!}{%
\begin{tikzpicture}[
  graynode/.style={circle,fill=gray!48,inner sep=2.25pt},
  graywire/.style={draw=gray!45,line width=.65pt},
  deadwire/.style={draw=black,line width=.95pt},
  hotwire/.style={draw=black,line width=1.05pt},
  deadnode/.style={circle,draw=black,very thick,fill=white,inner sep=2.45pt},
  graybox/.style={
    draw=gray!48,
    line width=.62pt,
    fill=white,
    minimum size=.42cm,
    inner sep=0pt,
    path picture={
      \draw[gray!48,line width=.55pt]
        (-.15,-.08) .. controls (-.07,.12) and (.07,.12) .. (.15,-.08);
    }
  },
  grayboxb/.style={
    draw=gray!48,
    line width=.62pt,
    fill=white,
    minimum size=.42cm,
    inner sep=0pt,
    path picture={
      \draw[gray!48,line width=.55pt]
        (-.15,.08) .. controls (-.07,-.12) and (.07,-.12) .. (.15,.08);
    }
  },
  grayboxc/.style={
    draw=gray!48,
    line width=.62pt,
    fill=white,
    minimum size=.42cm,
    inner sep=0pt,
    path picture={
      \draw[gray!48,line width=.55pt] (-.15,-.09) -- (.15,.09);
    }
  },
  flatzerobox/.style={
    draw=black,
    very thick,
    fill=white,
    minimum size=.50cm,
    inner sep=0pt,
    path picture={
      \draw[black,line width=.78pt]
        (-.21,.02) -- (-.10,.02)
        .. controls (-.06,-.05) and (.06,-.05) .. (.10,.02)
        -- (.21,.02);
      \draw[black!55,line width=.58pt,densely dashed]
        (-.085,-.09) .. controls (-.04,.15) and (.04,.15) .. (.085,-.09);
    }
  },
  paneltitle/.style={font=\small\bfseries,align=center},
  panelnote/.style={font=\scriptsize,align=center,text=black},
  witnesslabel/.style={font=\scriptsize,align=center,text=black}
]
  \FisherNullMLPPanel{-4.20}{0}{.88}
  \FisherNullKANPanel{0}{0}{.88}{kanA}{flatzerobox}{B. zero coefficient, exposed}{edge reached\\path transmits}{hotwire}
  \FisherNullKANPanel{4.20}{0}{.88}{kanB}{flatzerobox}{C. zero coefficient, unexposed}{no effective path\\score column zero}{graywire}
\end{tikzpicture}}
\caption{MLPs and fixed-basis KANs turn visible simplicity into Fisher geometry
in different ways.  In a multilayer MLP, a dead ReLU hidden unit closes the
gate on the task support: both the activation and its local derivative vanish,
so the unit's score columns are zero (panel A).  In a KAN, Fisher visibility is
graph-path based and has two factors: the data must reach a basis atom's support
(upstream coverage), and the edge must transmit to the output (downstream
transmission).  Together these define the effective exposure introduced in
Section~\ref{sec:fixed-fisher-geom}.  Accordingly, only panel B highlights an
active input-to-output tangent path.  A zero coefficient alone does not certify
that this path is cut: if the atom is still reached and transmitting, an
infinitesimal coefficient change remains visible to the Fisher matrix (panel
B).  If either coverage or transmission is cut, the score column vanishes and
gives an exact Fisher null (panel C).}
\label{fig:fisher-null-events}
\end{figure}
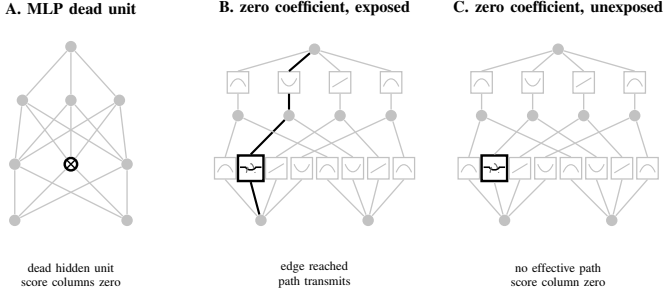

\bigskip

\paragraph*{Contributions}
\begin{enumerate}
\item We separate parameter magnitude from path closure: a dead ReLU activation
makes its associated parameter block Fisher-null, while a zero fixed-basis KAN
coefficient need not close its path or null its own score direction.
\item We derive the Fisher Gram representation for fixed-basis KANs and show
that their Fisher simplicity is graph-path based.  Under local dictionary
independence and effective-measure nondegeneracy, zero effective exposure
exactly characterizes edge-block Fisher nulls.
\item We illustrate the predicted local Fisher behavior in controlled synthetic
edge diagnostics and derive scoped implications for pruning, without claiming
end-to-end pruning gains.
\end{enumerate}

\bigskip

\paragraph*{Outline}
We derive the fixed-basis KAN Fisher identities in
Section~\ref{sec:fixed-fisher-geom}: the Gaussian Gram identity in the
single-layer case, and the edge-specific effective distribution in the
multilayer case.  We then compare Fisher-null mechanisms in KANs and ReLU MLPs
in Section~\ref{sec:exact-nulls}.  In Section~\ref{sec:pruning}, we apply the
exposure criterion to KAN pruning, test whether coefficient magnitude predicts
KAN-edge FIM effect, and define the retained-dictionary Fisher used by the
pruning scores.
We conclude in Section~\ref{sec:conclusion}.

\section{Fisher Geometry of Fixed-Basis KANs}
\label{sec:fixed-fisher-geom}

We express the Fisher geometry of a fixed-basis KAN in terms of task-weighted
flow through its computational graph, as illustrated in
Figure~\ref{fig:fisher-null-events}.  In a single layer, the input distribution
determines which basis atoms are active.  In multiple layers, the tangent
feature of a spline coefficient factors into upstream atom activation and
downstream transmission to the output.  We encode these two factors in an
effective edge distribution.  In Section~\ref{sec:exact-nulls}, we show that,
under local linear independence of the active dictionary functions and
effective-measure nondegeneracy on knot intervals, zero effective exposure
exactly characterizes the edge-block Fisher null directions.

\subsection{Single layer and predictor Jacobians}
\label{subsec:fixed-fisher-geom-single}

A one-layer fixed-basis KAN maps $\R^n \to \R$ via
\begin{equation}
\label{eq:kan}
f_\theta(x) = \sum_{p=1}^{n} \sum_{m=1}^{d} \theta_{p,m}\, b_m(x_p),
\end{equation}
where $\{b_m\}_{m=1}^d$ is the fixed dictionary, consisting of B-spline atoms
and, optionally, one SiLU atom, and $\theta_{p,m}$ are learnable coefficients
($D = nd$ total).  Note that the model is linear in~$\theta$, as
the feature map
$\phi(x)=(b_m(x_p))_{p=1,\ldots,n;\,m=1,\ldots,d}\in\R^D$
is independent of~$\theta$.

\bigskip

Let
$\eta_\theta:\mathcal X\to\R$ be a scalar predictor on the input space
$\mathcal X$ (here $\mathcal X=\R^n$), and let $q_\eta$ be a regular scalar
observation family with score
$s(y;\eta)=\partial_\eta\log q_\eta(y)$ and scalar Fisher information
\[
w(\eta)=\E_{Y\sim q_\eta}[s(Y;\eta)^2].
\]
Thus $w(\eta)$ is the per-observation information, or likelihood curvature, at
prediction value~$\eta$.  For the conditional model
$p_\theta(y\mid x)=q_{\eta_\theta(x)}(y)$, write
$J_\theta(x)=\nabla_\theta\eta_\theta(x)$.

\begin{proposition}[Fixed-basis Gaussian Gram reduction]
\label{prop:gram}
For the fixed-basis KAN coefficient model
$\eta_\theta(x)=\theta^\top\phi(x)$ under the fixed-variance Gaussian
observation model $q_\eta=\mathcal N(\eta,\sigma^2)$,
\begin{equation}
\label{eq:gram}
\FIM_\phi = \sigma^{-2}\,\E[\phi(X)\phi(X)^\top].
\end{equation}
In particular, $\FIM_\phi$ depends only on the basis and the input
distribution~$P_X$, not on the learned coefficients~$\theta$.
\end{proposition}

\begin{proof}
The conditional score is
\[
\nabla_\theta\log p_\theta(Y\mid X)
=\sigma^{-2}(Y-\theta^\top\phi(X))\phi(X).
\]
Since
\(\E[(Y-\theta^\top\phi(X))^2\mid X]=\sigma^2\),
averaging the score outer product gives~\eqref{eq:gram}.  The formula contains
\(\phi\) and \(P_X\), but not~\(\theta\).
\end{proof}

\medskip

To extend this further,
we need a Fisher
identity that does not assume a linear coefficient model. 

\begin{proposition}[Predictor Jacobian Fisher identity]
\label{prop:induced-fisher}
For the conditional model above, the Fisher matrix is
\begin{equation}
\label{eq:induced-fisher}
\FIM(\theta)
=\E_X\!\left[
  w(\eta_\theta(X))J_\theta(X)J_\theta(X)^\top
\right].
\end{equation}
\end{proposition}

\begin{proof}
Since $P_X$ is independent of~$\theta$, the score of the conditional model is
\[
\nabla_\theta\log p_\theta(Y\mid X)
=s(Y;\eta_\theta(X))J_\theta(X)
\]
by the chain rule.  Conditioning on $X=x$,
\begin{align*}
&\E_{Y\sim q_{\eta_\theta(x)}}\!\left[
  \nabla_\theta\log p_\theta(Y\mid x)
  \nabla_\theta\log p_\theta(Y\mid x)^\top
\right] \\
&\qquad =
\E_{Y\sim q_{\eta_\theta(x)}}\!\left[
  s(Y;\eta_\theta(x))^2
\right]
J_\theta(x)J_\theta(x)^\top \\
&\qquad =
w(\eta_\theta(x))J_\theta(x)J_\theta(x)^\top .
\end{align*}
Averaging over~$X$ gives~\eqref{eq:induced-fisher}.
\end{proof}

\medskip

Thus the Fisher matrix is the observation-model information metric expressed in
parameter coordinates through the predictor's Jacobian.  Learned parameters can
enter only through $J_\theta$ or through the observation
curvature~$w(\eta_\theta)$.  In the fixed-basis Gaussian case,
$J_\theta=\phi$ and $w$ is constant, which is why
Proposition~\ref{prop:gram} is coefficient-independent; with nonconstant
observation curvature, the same identity gives a weighted Gram whose weights may
depend on~$\theta$.

\subsection{Multilayer}\label{subsec:fixed-fisher-geom-multi}

A multilayer fixed-basis KAN extends~\eqref{eq:kan} by stacking $L$ layers of
fixed-dictionary edge functions (Fig.~\ref{fig:fisher-null-events}), with
layer widths $n_0=n,n_1,\ldots,n_L=1$.  Let $h^{(0)}(x)=x$ and,
for layer~$\ell$,
\[
\begin{aligned}
h_q^{(\ell)}(x)
&=\sum_{p=1}^{n_{\ell-1}}\sum_{m=1}^{d}
\theta_{q,p,m}^{(\ell)}\,
b_m\!\left(h_p^{(\ell-1)}(x)\right),\\
&\hspace{2.5em} q=1,\ldots,n_\ell .
\end{aligned}
\]
The scalar output is $\eta_\theta(x)=h_1^{(L)}(x)$, up to any final readout.  A
hidden-edge block differs from the one-layer Gram identity only through the
distribution seen by the edge: upstream activations supply the edge input, and
the downstream derivative reweights it.  This gives the edge exposure criterion
used below.

\bigskip

Fix one scalar fixed-dictionary edge~$e$ in a trained network.  Let $H_e(X)$ be the scalar
input reaching that edge, and let
\(b_1,\ldots,b_G\) be the fixed dictionary atoms on this edge, with
\(b_e(z)=(b_1(z),\ldots,b_G(z))^\top\).  Indices \(i,j\) label edge-basis atoms
and \(e\) labels the selected edge.  The edge output is
\[
u_e(X)=\sum_{i=1}^G\theta_i b_i(H_e(X))
\]
with coefficients
$\theta_e=(\theta_1,\ldots,\theta_G)$; all other network parameters are held
fixed while forming the edge block.  Let
\[
W_e(X)=\frac{\partial\eta_\theta(X)}{\partial u_e(X)}
\]
be the downstream derivative of the final scalar prediction with respect to this
edge output.  The downstream-weighted exposure is
\begin{equation}
\label{eq:edge-weight}
R_e(X)=w(\eta_\theta(X))W_e(X)^2 .
\end{equation}
The factor $W_e$ is the downstream transmission term: it is nonzero when
perturbations of this edge can change the final prediction.
This weight defines the unnormalized effective edge distribution
\begin{equation}
\label{eq:hidden-edge-distribution}
\nu_e(B)=\E\!\left[
  R_e(X)\mathbf{1}\{H_e(X)\in B\}
\right]
\end{equation}
for an edge-input region~$B$.  Let $D_{\mathrm{full}}$ be the total number of
network parameters, let $P_e:\R^G\to\R^{D_{\mathrm{full}}}$ embed an
edge-coefficient vector into that full parameter space, and let
\[
\FIM_{\mathrm{full}}
=\E\!\left[
  w(\eta_\theta(X))\nabla_\theta\eta_\theta(X)
  \nabla_\theta\eta_\theta(X)^\top
\right]
\]
be the Fisher matrix for all network parameters.  The effective exposure of the
$i$th basis atom is defined using
$\supp(b_i)=\overline{\{z:b_i(z)\ne0\}}$, the support of~$b_i$:
\begin{equation}
\label{eq:weighted-occupancy}
\begin{aligned}
\rho_i^{(e)}
  &= \E\!\left[
      w(\eta_\theta(X))W_e(X)^2
      \mathbf{1}\{H_e(X)\in\supp(b_i)\}
    \right] \\
  &= \nu_e(\supp(b_i)).
\end{aligned}
\end{equation}
Thus $\rho_i^{(e)}=0$ means that atom~$i$ is either not reached by the edge
input or not transmitted downstream on the reached region.

\begin{proposition}[Hidden-edge Fisher as an effective-distribution Gram]
\label{prop:hidden-edge-effective-distribution}
With the notation above, the edge tangent vector is
\begin{equation}
\label{eq:edge-tangent-vector}
J_e(X)
= \nabla_{\theta_e}\eta_\theta(X)
= W_e(X)b_e(H_e(X)).
\end{equation}
The Fisher block for the edge coefficients is the Gram matrix of the fixed
edge basis under~$\nu_e$:
\begin{equation}
\label{eq:edge-block-as-gram}
\begin{aligned}
\FIM_{ij}^{(e)}
  &= \E\!\left[
      w(\eta_\theta(X))W_e(X)^2
      b_i(H_e(X))b_j(H_e(X))
    \right] \\
  &= \int b_i(z)b_j(z)\,d\nu_e(z).
\end{aligned}
\end{equation}
For any $a\in\R^G$,
\begin{equation}
\label{eq:edge-kernel}
a\in\ker\FIM^{(e)}
\iff
\sum_i a_i b_i=0\quad \nu_e\text{-a.e.}
\end{equation}
For a coordinate direction, let $e_i\in\R^G$ denote the $i$th standard basis
vector.  The exposure certificate then says that if $\rho_i^{(e)}=0$, then
$b_i=0$ $\nu_e$-a.e., so
$e_i\in\ker\FIM^{(e)}$.
These edge null directions are full-network nulls:
\[
a\in\ker\FIM^{(e)}\quad\Longrightarrow\quad
P_ea\in\ker\FIM_{\mathrm{full}} .
\]
\end{proposition}

\begin{proof}
Apply Proposition~\ref{prop:induced-fisher} with the edge coefficients
$\theta_e$ as the parameter vector, holding all other network parameters fixed.
It remains to compute the edge Jacobian.  The chain rule gives
\begin{equation}
\label{eq:edge-tangent}
\frac{\partial\eta_\theta(X)}{\partial\theta_i}
=
\frac{\partial\eta_\theta(X)}{\partial u_e(X)}
\frac{\partial u_e(X)}{\partial\theta_i}
=
W_e(X)b_i(H_e(X)),
\end{equation}
which is \eqref{eq:edge-tangent-vector} in vector form.  Substituting this
Jacobian into~\eqref{eq:induced-fisher} gives the expectation form in
\eqref{eq:edge-block-as-gram}.  In matrix form,
\begin{equation}
\label{eq:edge-fisher-matrix}
\FIM^{(e)}
=
\E\!\left[
  R_e(X)b_e(H_e(X))b_e(H_e(X))^\top
\right].
\end{equation}
For any test function~$f$,
\[
\int f(z)\,d\nu_e(z)
=
\E\!\left[
  R_e(X)f(H_e(X))
\right].
\]
Substituting $f(z)=b_i(z)b_j(z)$ gives~\eqref{eq:edge-block-as-gram}.  For the
kernel statement,
\begin{align}
a^\top\FIM^{(e)}a
&=
\E\!\left[
  R_e(X)\left(\sum_i a_i b_i(H_e(X))\right)^2
\right] \notag\\
&=
\int\left(\sum_i a_i b_i(z)\right)^2\,d\nu_e(z).
\label{eq:edge-block-quadratic}
\end{align}
Since the integrand is nonnegative, \eqref{eq:edge-kernel} follows.  Finally,
\[
a^\top\FIM^{(e)}a=(P_ea)^\top\FIM_{\mathrm{full}}(P_ea).
\]
Because $\FIM_{\mathrm{full}}$ is positive semidefinite, a zero quadratic form
implies $\FIM_{\mathrm{full}}P_ea=0$, proving the full-network null claim.
\end{proof}

\medskip

The shallow fixed-basis result is the special case $H_e=X$, $W_e\equiv1$, and
constant observation curvature.  Thus the multilayer case preserves the same
Gram structure, but replaces ordinary input coverage by effective exposure:
the atom must be reached by the edge input and transmitted downstream.

\section{Fisher Nulls}
\label{sec:exact-nulls}

We now pin down when ``null'' is exact.  A parameter is not Fisher-null
because its fitted value is small.  It is Fisher-null when its score column is
zero on the task support.  Fixed-basis KANs reach that condition through zero
effective exposure.  ReLU MLPs reach it through dead gates or homogeneity
symmetries.  The test is the same in both cases: does an infinitesimal parameter
change reach the predictive distribution on the data that matter?

\subsection{Fixed-basis KANs}

Section~\ref{sec:fixed-fisher-geom} already gave the easy direction:
$\rho_i^{(e)}=0$ makes coordinate~$i$ invisible to the edge Fisher.  For
B-spline edges, with or without one fixed SiLU atom, the converse holds once the
edge-input measure is not degenerate on knot intervals.  Then zero-exposure
atoms are the only coordinate directions left in the edge-block kernel.

\bigskip

From Proposition~\ref{prop:hidden-edge-effective-distribution}, the score
column for coefficient~$i$ on a fixed-basis KAN edge is
\[
\frac{\partial\eta_\theta(X)}{\partial\theta_i}
=W_e(X)b_i(H_e(X)).
\]
Let $\nu_e$ denote the effective edge distribution from
\eqref{eq:hidden-edge-distribution}.  Define the effective exposure by
\[
\rho_i^{(e)}
=\E\!\left[
  w(\eta_\theta(X))W_e(X)^2
  \mathbf{1}\{H_e(X)\in\supp(b_i)\}
\right].
\]
Thus $\rho_i^{(e)}=0$ implies $e_i\in\ker\FIM^{(e)}$ by
\eqref{eq:edge-kernel}.

Fix a nondecreasing knot sequence $(\tau_j)$.  An \emph{open knot interval} is
an interval $J=(\tau_j,\tau_{j+1})$ with $\tau_j<\tau_{j+1}$; its active
dictionary restrictions are the functions $b_i|_J$ that are not identically
zero.
We say that $\nu_e$ is
\emph{nondegenerate on knot intervals} if every exposed atom is active on some
open knot interval~$J$ with $\nu_e(J)>0$, and, on every such interval, any
linear combination of the active dictionary restrictions that vanishes
$\nu_e$-a.e.\ on~$J$ vanishes identically on~$J$.
This excludes edge-input measures concentrated on too few points to distinguish
the active atoms.
Adjoining one SiLU atom preserves local linear independence: the B-spline
restrictions are polynomial on each open knot interval, whereas SiLU is not
polynomial on any open interval.  Indeed, real-analytic continuation would
otherwise make SiLU polynomial on all of~$\R$, contradicting its decay to zero
as its argument tends to~$-\infty$.
The resulting exact rank characterization is as follows.

\begin{theorem}[Exact edge rank via effective disconnects]
\label{thm:exact-edge-rank}
Assume that the active dictionary restrictions are locally linearly independent
on knot intervals~\cite{deBoor1978,schumaker2007}, and that~$\nu_e$ is nondegenerate
on knot intervals.  Then, for every edge-coefficient direction $a\in\R^G$,
\begin{equation}
\label{eq:edge-kernel-iff}
a\in\ker\FIM^{(e)}
\quad\Longleftrightarrow\quad
\supp(a)\subseteq \{i:\rho_i^{(e)}=0\}.
\end{equation}
In particular,
\begin{equation}
\label{eq:coordinate-kernel-iff}
e_i\in\ker\FIM^{(e)}
\quad\Longleftrightarrow\quad
\rho_i^{(e)}=0,
\end{equation}
and therefore
\begin{align*}
\ker\FIM^{(e)}
  &= \operatorname{span}\{e_i:\rho_i^{(e)}=0\}, \\
\rank\FIM^{(e)}
  &= \#\{i:\rho_i^{(e)}>0\}.
\end{align*}
Thus the edge block is full rank exactly when every basis atom is effectively
connected.
\end{theorem}

Here is the main step in the converse.  If $a\in\ker\FIM^{(e)}$, then
$\sum_i a_i b_i=0$ $\nu_e$-almost everywhere by \eqref{eq:edge-kernel}.  On an
exposed atom, nondegeneracy turns that almost-sure equality into an identity on
an active knot interval.  Local linear independence then forces the coefficient
to vanish.  The reverse implication is immediate because every zero-exposure
atom vanishes $\nu_e$-almost everywhere.  We give the full proof and the
excluded pathologies in Appendix~\ref{app:exact-edge-rank}.  The hypotheses are
doing real work: a measure on $M<G$ points can expose all $G$ atoms while the
Gram rank is at most~$M$, and adding a constant atom to a partition-of-unity
spline dictionary creates a structural null when the effective measure is
supported on the dictionary's basic interval.

\medskip

The theorem is only about exact nullity.  Small positive exposure is still
positive Fisher curvature; it is not a rank-free deletion certificate.

\subsection{ReLU MLPs}

For ReLU MLPs, exact Fisher nulls come from unit inactivity or parameter
symmetry, not from a small weight by itself.  A dead hidden unit zeros both the
activation and the local derivative on the task support, so the associated
score directions vanish in any hidden layer.  Figure~\ref{fig:fisher-null-events}
shows the mechanism inside a multilayer MLP; for the local algebra, it suffices
to write the scalar-output one-hidden-layer model as
$f(x)=\sum_k v_k r(w_k^\top x+b_k)$, where $r(z)=\max(0,z)$,
$w_k\in\R^n$ and $b_k$ are the incoming weight and bias, and $v_k$ is the
outgoing weight of unit~$k$.

\begin{proposition}[ReLU mechanisms create exact null directions]
\label{prop:dead-neuron}
For this ReLU MLP:
\begin{enumerate}
\item If hidden unit~$j$ satisfies $w_j^\top X+b_j\leq0$ almost surely and
$\Prob\{w_j^\top X+b_j=0\}=0$, then the score with respect to all $n+2$
parameters of unit~$j$ vanishes almost surely.
\item If an active unit satisfies $P(w_j^\top X + b_j = 0) = 0$, then the
rescaling $(w_j,b_j,v_j)\mapsto(cw_j,cb_j,v_j/c)$ for $c>0$ preserves the
function, and the tangent vector $\eta_j=(w_j,b_j,-v_j)$ lies in
$\ker(\FIM)$.
\end{enumerate}
\end{proposition}

\begin{proof}
With $w_j^\top X+b_j<0$ almost surely,
\begin{align*}
\frac{\partial f}{\partial w_{j,i}}
&= v_jr'(\cdot)x_i = 0,\\
\frac{\partial f}{\partial b_j}
&= v_jr'(\cdot) = 0,\\
\frac{\partial f}{\partial v_j}
&= r(\cdot) = 0 .
\end{align*}
All $n+2$ score directions vanish.

Positive homogeneity gives, for $c>0$,
\[
\frac{v_j}{c}r\!\left(c(w_j^\top x + b_j)\right)
=
v_jr(w_j^\top x + b_j).
\]
Differentiating at $c=1$ gives
\[
w_j^\top\nabla_{w_j}f
+ b_j\nabla_{b_j}f
- v_j\nabla_{v_j}f
=0
\quad P_X\text{-a.s.}
\]
The statistical-model score in direction $\eta_j$ therefore vanishes almost
surely, so $\eta_j\in\ker(\FIM)$.  The nonredundant coefficient
parameterization studied here has no separate scale coordinate and hence no
analogous scale-coefficient symmetry.  If an implementation parameterizes a
spline term as $\alpha_e\sum_i\theta_i b_i$ with both $\alpha_e$ and $\theta$
trainable, it has the rescaling null
$(\alpha_e,\theta)\mapsto(c\alpha_e,\theta/c)$.
\end{proof}

\medskip

Dead or redundant ReLU units create positive-dimensional parameter sets with
zero Kullback-Leibler (KL) divergence, so their rank loss is a singular
phenomenon.  Proposition~\ref{prop:dead-neuron} is the MLP counterpart of the
KAN zero-exposure condition in Theorem~\ref{thm:exact-edge-rank}.  Unit
inactivity sets the relevant score multiplier to zero on the task support; zero
effective exposure removes a KAN basis direction from the edge Fisher.  In both
architectures, Fisher rank loss is defined by a vanishing score direction.
Parameter magnitude is not the defining criterion.
Indeed, zero parameters can create other null directions indirectly in either
architecture by cutting upstream coverage or downstream transmission.  For
example, a zero MLP outgoing weight nulls that unit's incoming-weight and bias
directions, while a downstream KAN coefficient pattern whose derivative
vanishes can null directions on an upstream edge.  Thus the contrast is between
path closure and magnitude alone, not between zero parameters in KANs and MLPs.

\FloatBarrier
\section{Consequences for KAN Pruning}
\label{sec:pruning}

We ask when a fixed-basis KAN atom can be removed without paying
Fisher rank.  The exact answer is the zero-exposure case.  If an atom has zero
exposure on the current edge block, deleting its coordinate does not reduce that
block's rank.  A positive-exposure atom is different: deleting it changes the
retained dictionary, so the smaller model has to be scored on its own.  The
numerical check below asks a separate question.  Once exposure is known, does
the fitted coefficient magnitude still predict the FIM effect of a KAN basis
atom?
The code, fixed seeds, and archived per-seed outputs for these numerical checks
are available at \url{https://github.com/atavory/fisher_kan} (artifact commit
\texttt{5ebd51d}).

\bigskip

Recall from Section~\ref{sec:fixed-fisher-geom} that, for a hidden KAN edge,
Proposition~\ref{prop:hidden-edge-effective-distribution} rewrites the
coefficient score as
\[
\frac{\partial\eta_\theta(X)}{\partial\theta_i}
=W_e(X)b_i(H_e(X)).
\]
This motivates the effective exposure
\[
\rho_i^{(e)}
=\E\!\left[w(\eta_\theta(X))\,W_e(X)^2\,
   \mathbf{1}\{H_e(X)\in\supp(b_i)\}\right].
\]
This is the weighted task mass on which atom~$i$ is reached by the edge input
and still transmitted downstream.  Section~\ref{sec:exact-nulls} showed that
$\rho_i^{(e)}=0$ certifies an exact edge-block null.  Under local dictionary
independence and effective-measure nondegeneracy on knot intervals,
Theorem~\ref{thm:exact-edge-rank} gives the converse.  For pruning, then,
$\rho_i^{(e)}$ is a preselection score for rank-free candidates before we fit
or score a retained dictionary.
For normalized B-spline atoms, $0\leq b_i\leq1$, so
\begin{equation}
\label{eq:curvature-exposure-envelope}
\FIM_{ii}^{(e)}
=\E\!\left[w(\eta_\theta(X))W_e(X)^2b_i(H_e(X))^2\right]
\leq \rho_i^{(e)}.
\end{equation}
Exposure is therefore a support-level upper envelope for coordinate curvature.
The zero set is the exact certificate.  A positive value is weaker: it says the
coordinate is present in the local statistical model, but it does not determine
the edge spectrum.

\bigskip

Pruning changes the dictionary, so the original Fisher block is not the whole
story.  After a deletion, the relevant Fisher matrix is the one for the
retained basis.  We state this matrix first because the scores below use it.

\begin{proposition}[Restricted Fisher for a pruned basis]
\label{prop:pruned-fisher}
In the fixed-variance Gaussian model of Proposition~\ref{prop:gram}, suppose a
set $R\subseteq\{1,\ldots,D\}$ of coefficients is retained and all other basis
directions are removed from the model class.  Then the Fisher matrix of the
pruned model is
\begin{equation}
\label{eq:restricted-fisher}
\FIM_R
= \sigma^{-2}\E[\phi_R(X)\phi_R(X)^\top]
= P_R^\top \FIM_\phi P_R,
\end{equation}
where $\phi_R=P_R^\top\phi$ and $P_R$ embeds the retained coordinates into the
original parameter space.
\end{proposition}

\begin{proof}
For the restricted model $f_{\theta_R}(x)=\theta_R^\top\phi_R(x)$, applying
Proposition~\ref{prop:gram} gives
$\FIM_R=\sigma^{-2}\E[\phi_R\phi_R^\top]$.  Since $\phi_R=P_R^\top\phi$, this is
$P_R^\top\FIM_\phi P_R$, the principal submatrix on the retained coordinates.
\end{proof}

\bigskip

When a deletion is not rank-free, pruning becomes a comparison among retained
dictionaries.  For a retained set~$R$ on an edge satisfying the hypotheses of
Theorem~\ref{thm:exact-edge-rank}, apply the theorem to the retained
subdictionary.  With $\FIM_R^{(e)}$ denoting its edge-block Fisher matrix,
\[
\rank \FIM_R^{(e)}
=\#\{i\in R:\rho_i^{(e)}>0\}.
\]
Zero-exposure removals are rank-free for the current edge block.  Positive
exposure removals delete directions with nonzero Fisher weight and require a
fit-complexity tradeoff.  This is still a statement about rank, not a guarantee
that the represented function is unchanged.  Pure noncoverage,
\[
\Prob\{H_e(X)\in\supp(b_i)\}=0,
\]
is lossless.  The atom and its derivative vanish on the task support, so
deleting it preserves the network function and the other exposures for any
coefficient value.  Downstream-flat zero exposure is weaker.  The atom is
reached, but $W_e=0$ on the reached region.  Deleting it is rank-free for the
current Fisher block, yet it can move the edge output outside the flat region
and change the network.  In an iterative procedure, exposures should be
recomputed after each accepted deletion because a function-changing deletion can
alter downstream derivatives and hence upstream exposures.

For a candidate retained dictionary~$R$, fix a Gaussian coefficient prior
$\theta\sim\mathcal N(0,\Sigma_0)$.  Let $N$ be the fitting-sample size and
$\widehat L_{\mathrm{val}}$ the validation negative log-likelihood.  The
restricted Gram and prior covariance are
$\Gram_R=\E[\phi_R\phi_R^\top]$ and
$\Sigma_{0,R}=P_R^\top\Sigma_0P_R$.  Define the prior-relative restricted
Fisher complexity
\[
\mathcal C_N(\Gram_R,\Sigma_{0,R})
=\frac12\log\det\!\left(
I+N\sigma^{-2}\Sigma_{0,R}^{1/2}\Gram_R\Sigma_{0,R}^{1/2}
\right).
\]
Here $I$ is the identity matrix on the retained coordinates.  We use the
Fisher-information-criterion (FIC) score
\[
S_{\mathrm{FIC}}(R)
= \widehat L_{\mathrm{val}}(R)
  + \frac{1}{N}\mathcal C_N(\Gram_R,\Sigma_{0,R}).
\]
This score keeps two quantities separate: validation loss handles approximation,
and the restricted Fisher handles complexity.  The rule we use is therefore
\[
\begin{aligned}
\rho_i^{(e)}=0
&:\quad \text{rank-free candidate in the current edge block},\\
&\quad \text{delete directly only in the noncoverage case},\\
\rho_i^{(e)}>0
&:\quad \text{nonzero-Fisher direction},\\
&\quad \text{evaluate deletion in the restricted model}.
\end{aligned}
\]

\bigskip

This leaves the empirical question for coefficient pruning: once exposure is
known, does a small fitted coefficient add information about FIM effect?
Figure~\ref{fig:exp-hidden-edge} uses a controlled fixed-basis KAN edge.
For each of 24 fixed seeds, we draw $N=40{,}000$ edge inputs from
$\operatorname{Beta}(1.4,3.0)$, evaluate $G=16$ uniformly knotted, clamped
cubic B-spline atoms and, for edge input~$h$, the fixed downstream sensitivity
$W_e(h)=0.5+|\sin(7h)|$.  We fit the edge coefficients by least squares to
$Y=\sum_i\theta_i^\star b_i(H_e)+\epsilon$, where the
$\theta_i^\star$ are independent $\mathcal N(0,1)$ draws and
$\epsilon\sim\mathcal N(0,0.15^2)$.  We compare each
fitted coefficient magnitude and effective exposure with the diagonal
edge-block FIM curvature induced by $W_e(H_e)b_i(H_e)$.  We compute Pearson
partial correlations by linear residualization across atoms using
$\log(\FIM_{ii}^{(e)}+10^{-12})$, then report their mean and a normal 95\%
interval across seeds; in the figure, $r$ denotes a Pearson correlation
coefficient.
The association between exposure and curvature is expected from
\eqref{eq:curvature-exposure-envelope}.  The coefficient-controlled comparison
asks whether fitted coefficient magnitude adds anything in this diagnostic.
All empirical inputs and targets are generated procedurally by the released
scripts; the experiments require no external dataset or accelerator.

\begin{figure}[t]
\centering
\includegraphics[width=\columnwidth]{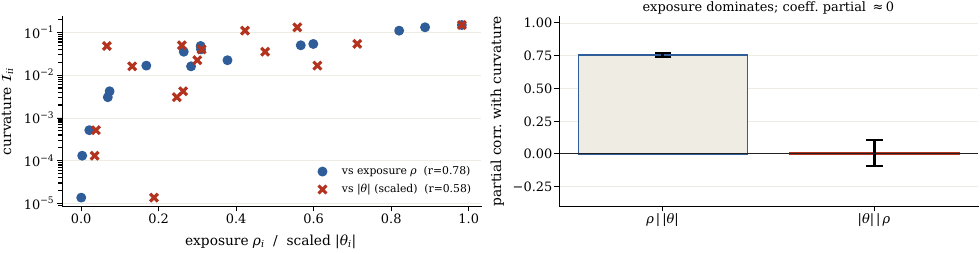}
\caption{Exposure, fitted coefficient magnitude, and edge-block curvature in
the controlled edge diagnostic.  Left: each point is one basis atom; diagonal
edge-block FIM curvature is plotted against effective exposure~$\rho_i^{(e)}$
and against fitted coefficient magnitude~$|\hat\theta_i|$ after rescaling for
comparability.  Right: Pearson partial correlation coefficients~$r$ over 24 seeds
with 95\% intervals.
Exposure remains predictive after controlling for fitted coefficient magnitude
($0.75\pm0.02$), while fitted coefficient magnitude carries no reliable
curvature information after controlling for exposure ($0.01\pm0.10$).}
\label{fig:exp-hidden-edge}
\end{figure}

Thus low-coefficient pruning is not, by itself, evidence that the deleted atom
had little FIM effect.  The effect has to be read from the exposed score column
or from the restricted Fisher of the retained dictionary.

\medskip

We also ran coefficient and disconnection controls over 24 fixed seeds, each
with $N=20{,}000$ task inputs.  For the $4\to5\to1$ ReLU network, inputs are
uniform on $[0,1]^4$, weights are Gaussian, and quantile-set biases keep all
five units active.  We compare zeroing one input weight with shifting that
unit's bias below every sampled preactivation.  For the KAN edge, inputs are
uniform on $[0,1]$, and $G=16$ clamped cubic B-spline atoms have Gaussian
coefficients.  We compare zeroing the central exposed coefficient with excluding
the first atom's support from the sampled inputs.  Each Jacobian is recomputed
after the intervention.  Numerical rank uses the cutoff
$\max(10^{-12},10^{-10}\lambda_{\max})$, where $\lambda_{\max}$ is the largest
eigenvalue; every reported rank is unchanged for relative tolerances from
$10^{-14}$ to $10^{-5}$.

\begin{table}[t]
\centering
\caption{Controlled Fisher-rank interventions.  The tested isolated parameter
zeros preserve rank, while closing the corresponding tangent paths removes the
predicted block directions.  Results are identical across 24 seeds.}
\label{tab:rank-controls}
\resizebox{\columnwidth}{!}{%
\begin{tabular}{@{}llrr@{}}
\toprule
Block & Intervention & Rank & $\Delta$ rank \\
\midrule
MLP hidden affine & Baseline & $25/25$ & n/a \\
MLP hidden affine & One active weight set to zero & $25/25$ & $0$ \\
MLP hidden affine & One ReLU made dead & $20/25$ & $-5$ \\
KAN edge & Baseline & $16/16$ & n/a \\
KAN edge & One exposed coefficient set to zero & $16/16$ & $0$ \\
KAN edge & One atom made unexposed & $15/16$ & $-1$ \\
\bottomrule
\end{tabular}%
}
\end{table}

\section{Discussion and Conclusion}
\label{sec:conclusion}

\subsection{Conclusion}
Path closure and Fisher simplicity coincide for a dead ReLU unit: the closed
activation region makes the associated parameter directions statistically
invisible on the task distribution.  A small or zero parameter value alone, in
either an MLP or a fixed-basis KAN, need not close a path or make its own score
direction Fisher-null.

Fisher simplicity in a fixed-basis KAN is graph-path based, not coefficient
based.  In the single-layer Gaussian case, the coefficient Fisher block is a
basis Gram matrix under the input distribution and does not contain the fitted
coefficients.  For a hidden edge, the effective distribution~$\nu_e$ combines
upstream coverage through $H_e$ with downstream transmission through
$R_e=w(\eta_\theta)W_e^2$.  Under the assumptions of
Theorem~\ref{thm:exact-edge-rank}, zero effective exposure exactly identifies
edge-block Fisher null directions.

The pruning consequence follows directly.  Zero-exposure atoms are rank-free
candidates, with lossless deletion guaranteed in the pure noncoverage case.
Positive-exposure deletions change the restricted Fisher and must be judged in
the retained model.  Coefficient magnitude alone is not a Fisher-based pruning
criterion.  Our controlled experiments diagnose this local Fisher geometry;
they do not establish improved end-to-end pruning performance.

\subsection{Related Work}
KANs were introduced as networks whose edges carry learned univariate
functions~\cite{liu2024kan}, motivated by the Kolmogorov-Arnold
representation theorem~\cite{kolmogorov1957,arnold1957}.  Their interpretability
program uses visualized edge functions, sparsification and pruning, and symbolic
formula extraction; KAN 2.0 broadens it to feature attribution, modular
structure, and scientific discovery~\cite{liu2024kan,liu2025kan2}.  Their
standard implementation uses spline dictionaries, so classical spline
approximation theory~\cite{deBoor1978,schumaker2007} supplies the fixed-basis
language used here.  The spline input is the classical local-independence and
Schoenberg-Whitney collocation lineage; the KAN-specific step is to move that
language to the effective edge distribution~$\nu_e$.  Recent KAN
generalization work studies capacity through standard complexity
measures~\cite{zhang2024}.  Our focus is the local Fisher geometry behind such
measures.  Other KAN analyses study expressiveness and spectral bias
~\cite{wang2025spectral} or relate piecewise-linear KANs to ReLU
networks~\cite{schoots2025relu}; these results are complementary to the
invariant versus representation-dependent distinction studied here.  The
finite-sample complexity term follows the spectrum-based regret framework of
Feder, Urbanke, and Fogel~\cite{feder2025framework}, with roots in stochastic
complexity and effective parameter counting~\cite{rissanen1996,mackay1992}.
Fisher information defines the local Riemannian metric of the statistical
manifold~\cite{amari1998}, and Fisher spectra of neural networks have been
studied empirically and theoretically~\cite{karakida2019,sagun2017}.  The
rank-loss side connects to singular learning theory~\cite{watanabe2009}, dead
and inactive-unit pruning~\cite{hu2016trimming,lu2020dying}, rescaling
symmetries and sharpness pathologies of ReLU networks
\cite{dinh2017,neyshabur2015}, and the contrast between fixed-feature kernel
regimes~\cite{jacot2018,chizat2019} and feature-learning models.

\subsection{Limitations and Open Questions}
The scope is the nonredundant fixed-dictionary coefficient parameterization: a
fixed B-spline grid, optionally with one SiLU atom.  If an implementation adds a
trainable spline-scale multiplier, it also adds a scale-coefficient rescaling
null that the coefficient block alone does not capture.  Adaptive-grid and
learned-basis KANs make the basis parameter-dependent, so the
coefficient-independent Gram of Proposition~\ref{prop:gram} no longer applies
without recomputing the Fisher identity.  Depth adds another dependence.  Even
with fixed basis atoms, $\nu_e$ changes with learned upstream activations and
downstream derivatives; coefficient-independence should therefore be read
edge-locally and at the current network.  Pruning is also not a complete
definition of complexity.  Restricting the Fisher presumes that a sub-model has
already been chosen, rank is generically full, and covariant log-det complexity
still depends on a prior.  The invariant versus representation-dependent split
is a constraint on any representation-aware complexity definition, not the
definition itself.

\appendices
\section{Exact B-spline edge rank}
\label{app:exact-edge-rank}
\label{app:sharpness}

For a fixed B-spline edge, effective exposure gives an exact rank
characterization once two pathologies are excluded.  First, the effective
distribution must not collapse onto isolated roots of a local spline
polynomial.  Second, the dictionary itself must not contain structural linear
dependencies.  We prove Theorem~\ref{thm:exact-edge-rank} and then show why
its hypotheses are not cosmetic.

\begin{proof}[Proof of sufficiency in Theorem~\ref{thm:exact-edge-rank}]
Let $g_a(z)=\sum_i a_i b_i(z)$.  The quadratic form of the edge block is
\begin{align}
\label{eq:appendix-edge-kernel-form}
a^\top\FIM^{(e)}a
  &= \int g_a(z)^2\,d\nu_e(z),\\
a\in\ker\FIM^{(e)}
  &\Longleftrightarrow
  g_a=0\quad \nu_e\text{-a.e.} \notag
\end{align}
The equivalence uses nonnegativity of the integrand.
Assume $\supp(a)\subseteq\{i:\rho_i^{(e)}=0\}$.  For each coordinate in the
support of~$a$, $\rho_i^{(e)}=0$ implies $b_i=0$ $\nu_e$-a.e.  Hence
\begin{align*}
g_a(z)
  &= \sum_{i:\rho_i^{(e)}=0} a_i b_i(z)\\
  &=0\quad \nu_e\text{-a.e.},
\end{align*}
and~\eqref{eq:appendix-edge-kernel-form} gives $a\in\ker\FIM^{(e)}$.
\end{proof}

The converse is the essential step.  Positive exposure alone does not rule out
coupled cancellations among exposed atoms, and the dictionary itself can
introduce data-invariant nulls.  The next two examples show why the hypotheses
in Theorem~\ref{thm:exact-edge-rank} are not cosmetic.

\begin{example}[Discrete state collapse]
\label{ex:discrete-state-collapse}
Suppose the effective distribution is carried by $M$ points
$z_1,\dots,z_M$.  Write the collocation matrix
$B=[\,b_i(z_m)\,]\in\R^{M\times G}$ and
$D=\diag\!\big(\nu_e(\{z_m\})\big)$.  By~\eqref{eq:edge-fisher-matrix},
\begin{align}
\label{eq:atomic-gram}
\FIM^{(e)}
  &= \sum_{m=1}^{M}\nu_e(\{z_m\})\,b_e(z_m)b_e(z_m)^\top \notag\\
  &= B^\top D B,\\
\rank\FIM^{(e)}
  &= \rank(B^\top D B)
   \leq \rank B
   \leq M . \notag
\end{align}
Thus every vector in $\ker B$ is an edge-Fisher null:
\begin{align*}
Ba=0
&\quad\Longrightarrow\quad
a^\top\FIM^{(e)}a
  = a^\top B^\top D B a\\
&\quad\Longrightarrow\quad
a^\top\FIM^{(e)}a=0 .
\end{align*}
A cubic atom is nonzero on four consecutive knot spans, so $M$ points can meet
all $G$ supports even when $M<G$.  Choosing the points in span interiors gives
\begin{align*}
\rho_i^{(e)}
  &= \nu_e(\supp b_i)>0
  \qquad \text{for all } i,\\
\dim\ker\FIM^{(e)}
  &= G-\rank\FIM^{(e)}
   \geq G-M
   >0 .
\end{align*}
For $G=9$ and $M=6$, all nine atoms can have positive exposure while
$\rank\FIM^{(e)}\le6$.  The missing dimensions are coupled cancellations among
exposed atoms, not coordinate directions explained by zero exposure.  Positive
occupancy alone therefore does not imply full rank; the interval
nondegeneracy condition rules out exactly this discrete-state pathology.
\end{example}

\begin{example}[A structural dictionary null]
\label{ex:structural-dictionary-null}
The second failure mode is a dictionary dependency.  B-spline atoms obey the
partition of unity $\sum_{i=1}^{G}b_i\equiv1$ on the basic interval of the knot
sequence.  If one adjoins a constant atom $b_0\equiv1$ and takes
$a=(a_0,a_1,\dots,a_G)=(1,-1,\dots,-1)$, then
$\sum_{i=0}^{G}a_i b_i=1-\sum_{i=1}^{G}b_i\equiv0$ on that interval.  Hence,
for every effective distribution~$\nu_e$ supported there,
\begin{equation}
\label{eq:pou-null}
a^\top\FIM^{(e)}a=\int\Big(\sum_{i=0}^{G}a_i b_i\Big)^{2}d\nu_e=0 ,
\end{equation}
an exact null present under every such input distribution.  This is a
structural, data-invariant null, analogous to the MLP symmetry null in
Proposition~\ref{prop:dead-neuron}, not a missing-exposure event.  Adjoining one
SiLU atom does not create such a dependency: on each open knot interval the
spline restrictions are polynomial, whereas SiLU is not polynomial.  Thus the
augmented fixed dictionary remains locally linearly independent, and
Theorem~\ref{thm:exact-edge-rank} applies under its effective-measure
nondegeneracy condition.
\end{example}

\begin{proof}[Proof of necessity in Theorem~\ref{thm:exact-edge-rank}]
Assume $a\in\ker\FIM^{(e)}$ and fix any exposed atom $i$ with
$\rho_i^{(e)}>0$.  Let $g_a(z)=\sum_i a_i b_i(z)$.  By~\eqref{eq:edge-kernel},
$g_a=0$ $\nu_e$-a.e.; equivalently,
\begin{align}
0
&=a^\top\FIM^{(e)}a \notag\\
&=\int g_a(z)^2\,d\nu_e(z).
\label{eq:appendix-necessity-zero}
\end{align}
Positive exposure means that some knot interval~$J$ on which $b_i$ is active
has $\nu_e(J)>0$:
\begin{align*}
\rho_i^{(e)}>0
&\Longrightarrow
\exists J\ \text{with}\ b_i|_J\not\equiv0,\ \nu_e(J)>0 .
\end{align*}
Since~$\nu_e$ is nondegenerate on knot intervals,
\eqref{eq:appendix-necessity-zero} makes $g_a$ vanish identically on~$J$.
Therefore
\[
0=g_a|_J=\sum_{r:\,b_r|_J\not\equiv0} a_r\, b_r|_J .
\]
Local linear independence of the active dictionary restrictions forces
$a_r=0$ for every active~$r$, in particular $a_i=0$.  Since this holds for
every exposed atom, $\supp(a)\subseteq\{i:\rho_i^{(e)}=0\}$.  Together with the
sufficiency direction, this proves
\begin{align*}
\ker\FIM^{(e)}
&=\{a:a_i=0\ \text{whenever}\ \rho_i^{(e)}>0\}\\
&=\operatorname{span}\{e_i:\rho_i^{(e)}=0\}.
\end{align*}
The coordinate criterion~\eqref{eq:coordinate-kernel-iff} is the special case
$a=e_i$, and the rank formula follows by rank-nullity:
\[
\rank\FIM^{(e)}=G-\dim\ker\FIM^{(e)}
  =\#\{i:\rho_i^{(e)}>0\}.
\]
\end{proof}

\bibliographystyle{IEEEtran}
\bibliography{references}

@inproceedings{liu2024kan,
  title = {{KAN}: Kolmogorov-Arnold Networks},
  author = {Liu, Ziming and Wang, Yixuan and Vaidya, Sachin and Ruehle, Fabian and Halverson, James and Solja{\v{c}}i{\'c}, Marin and Hou, Thomas Y and Tegmark, Max},
  booktitle = {International Conference on Learning Representations},
  year = {2025}
}

@article{liu2025kan2,
  title = {Kolmogorov-Arnold Networks Meet Science},
  author = {Liu, Ziming and Tegmark, Max and Ma, Pingchuan and Matusik, Wojciech and Wang, Yixuan},
  journal = {Physical Review X},
  volume = {15},
  number = {4},
  pages = {041051},
  year = {2025},
  doi = {10.1103/4t7t-v19l}
}

@misc{feder2025framework,
  title = {Information-Theoretic Framework for Understanding Modern Machine-Learning},
  author = {Feder, Meir and Urbanke, R{\"u}diger and Fogel, Yaniv},
  year = {2025},
  eprint = {2506.07661},
  archiveprefix = {arXiv},
  note = {arXiv:2506.07661}
}

@article{kolmogorov1957,
  title = {On the Representation of Continuous Functions of Many Variables by Superposition of Continuous Functions of One Variable and Addition},
  author = {Kolmogorov, A. N.},
  journal = {Doklady Akademii Nauk SSSR},
  volume = {114},
  number = {5},
  pages = {953--956},
  year = {1957}
}

@article{arnold1957,
  title = {On Functions of Three Variables},
  author = {Arnold, V. I.},
  journal = {Doklady Akademii Nauk SSSR},
  volume = {114},
  pages = {679--681},
  year = {1957}
}

@book{deBoor1978,
  title = {A Practical Guide to Splines},
  author = {de Boor, Carl},
  publisher = {Springer},
  year = {1978}
}

@article{rissanen1996,
  title = {Fisher Information and Stochastic Complexity},
  author = {Rissanen, Jorma},
  journal = {IEEE Transactions on Information Theory},
  volume = {42},
  number = {1},
  pages = {40--47},
  year = {1996}
}

@article{mackay1992,
  title = {A Practical Bayesian Framework for Backpropagation Networks},
  author = {MacKay, David J. C.},
  journal = {Neural Computation},
  volume = {4},
  number = {3},
  pages = {448--472},
  year = {1992}
}

@book{watanabe2009,
  title = {Algebraic Geometry and Statistical Learning Theory},
  author = {Watanabe, Sumio},
  publisher = {Cambridge University Press},
  year = {2009}
}

@misc{zhang2024,
  title = {Generalization Bounds and Model Complexity for Kolmogorov-Arnold Networks},
  author = {Zhang, X. and Zhou, H.},
  year = {2024},
  eprint = {2410.08026},
  archiveprefix = {arXiv},
  note = {arXiv:2410.08026}
}

@inproceedings{wang2025spectral,
  title = {On the Expressiveness and Spectral Bias of {KANs}},
  author = {Wang, Yixuan and Siegel, Jonathan W. and Liu, Ziming and Hou, Thomas Y.},
  booktitle = {International Conference on Learning Representations},
  year = {2025}
}

@inproceedings{schoots2025relu,
  title = {Relating Piecewise Linear Kolmogorov Arnold Networks to {ReLU} Networks},
  author = {Schoots, Nandi and Villani, Mattia Jacopo and uit de Bos, Niels},
  booktitle = {Proceedings of the International Conference on Artificial Intelligence and Statistics},
  series = {Proceedings of Machine Learning Research},
  volume = {258},
  year = {2025}
}

@inproceedings{karakida2019,
  title = {Universal Statistics of Fisher Information in Deep Neural Networks: Mean Field Approach},
  author = {Karakida, Ryo and Akaho, Shotaro and Amari, Shun-ichi},
  booktitle = {Proceedings of the International Conference on Artificial Intelligence and Statistics},
  year = {2019}
}

@article{amari1998,
  title = {Natural Gradient Works Efficiently in Learning},
  author = {Amari, Shun-ichi},
  journal = {Neural Computation},
  volume = {10},
  number = {2},
  pages = {251--276},
  year = {1998}
}

@misc{sagun2017,
  title = {Empirical Analysis of the Hessian of Over-Parametrized Neural Networks},
  author = {Sagun, Levent and Evci, Utku and G{\"u}ney, V. Ugur and Dauphin, Yann N. and Bottou, L{\'e}on},
  year = {2017},
  eprint = {1706.04454},
  archiveprefix = {arXiv},
  note = {arXiv:1706.04454}
}

@book{schumaker2007,
  title = {Spline Functions: Basic Theory},
  author = {Schumaker, Larry L.},
  edition = {3},
  publisher = {Cambridge University Press},
  year = {2007}
}

@inproceedings{dinh2017,
  title = {Sharp Minima Can Generalize for Deep Nets},
  author = {Dinh, Laurent and Pascanu, Razvan and Bengio, Samy and Bengio, Yoshua},
  booktitle = {International Conference on Machine Learning},
  year = {2017}
}

@inproceedings{neyshabur2015,
  title = {Norm-Based Capacity Control in Neural Networks},
  author = {Neyshabur, Behnam and Tomioka, Ryota and Srebro, Nathan},
  booktitle = {Conference on Learning Theory},
  year = {2015}
}

@inproceedings{jacot2018,
  title = {Neural Tangent Kernel: Convergence and Generalization in Neural Networks},
  author = {Jacot, Arthur and Gabriel, Franck and Hongler, Cl{\'e}ment},
  booktitle = {Advances in Neural Information Processing Systems},
  year = {2018}
}

@inproceedings{chizat2019,
  title = {On Lazy Training in Differentiable Programming},
  author = {Chizat, L{\'e}na{\"i}c and Oyallon, Edouard and Bach, Francis},
  booktitle = {Advances in Neural Information Processing Systems},
  year = {2019}
}

@misc{hu2016trimming,
  title = {Network Trimming: A Data-Driven Neuron Pruning Approach towards Efficient Deep Architectures},
  author = {Hu, Hengyuan and Peng, Rui and Tai, Yu-Wing and Tang, Chi-Keung},
  year = {2016},
  eprint = {1607.03250},
  archiveprefix = {arXiv},
  note = {arXiv:1607.03250}
}

@article{lu2020dying,
  title = {Dying {ReLU} and Initialization: Theory and Numerical Examples},
  author = {Lu, Lu and Shin, Yeonjong and Su, Yanhui and Karniadakis, George Em},
  journal = {Communications in Computational Physics},
  volume = {28},
  number = {5},
  pages = {1671--1706},
  year = {2020},
  doi = {10.4208/cicp.OA-2020-0165}
}

\end{document}